\documentclass{article}

\usepackage[final]{neurips_2023}
\usepackage{multirow}

\usepackage[textsize=tiny]{todonotes}

\newcounter{mynotes}
\usepackage[utf8]{inputenc} 
\usepackage[T1]{fontenc}    
\usepackage{hyperref}       
\usepackage{url}            
\usepackage{booktabs}       
\usepackage{amsfonts}       
\usepackage{nicefrac}       
\usepackage{microtype}      
\usepackage{xcolor}         

\usepackage{microtype}
\usepackage{graphicx}
\usepackage{caption}
\usepackage{subcaption}
\usepackage{booktabs} 
\usepackage{xcolor}
\usepackage[normalem]{ulem} 
\usepackage{caption}

\usepackage{amsmath}
\usepackage{amssymb}
\usepackage{mathtools}
\usepackage{amsthm}

\usepackage[capitalize,noabbrev]{cleveref}

\usepackage{algorithm}
\usepackage[noend]{algorithmic}

\usepackage{hyperref}

\title{Optimization of Inter-group  Criteria for Clustering with
Minimum Size Constraints}

\author{%
  Eduardo S. Laber \\
  Department of Computer Science\\
  PUC-Rio\\
  Rio de Janeiro, RJ - Brazil \\
  \texttt{laber@inf.puc-rio.br} \\
  \And
  Lucas Murtinho \\
  Department of Computer Science \\
  PUC-Rio \\
  Rio de Janeiro, RJ - Brazil \\
  \texttt{lmurtinho@aluno.puc-rio.br} \\
}

\theoremstyle{plain}
\newtheorem{theorem}{Theorem}[section]
\newtheorem{proposition}[theorem]{Proposition}
\newtheorem{lemma}[theorem]{Lemma}

\theoremstyle{definition}

\theoremstyle{remark}

\theoremstyle{example}

\newtheorem{example}[theorem]{Example}

\newcommand{\spmst}{{\tt \mbox{MST-Sp}}}
\newcommand{\singlelink}{{\tt single-linkage}}

\newcommand{\spac}{{\tt spacing}}

\newcommand{\minsp}{{\tt \mbox{Min-Sp}}}

\newcommand{\dist}{{\tt dist}}
\newcommand{\C}{\mathcal{C}}
\newcommand{\X}{\mathcal{X}}
\newcommand{\A}{\mathcal{A}}

\usepackage[textsize=tiny]{todonotes}

\begin{document}
\maketitle

\begin{abstract}

Internal measures that are used to assess the quality of a clustering  usually take into account  intra-group and/or inter-group criteria. There are many papers in the literature that propose algorithms with provable approximation guarantees for optimizing  the former. However,  the optimization of inter-group criteria  is much less understood.

Here, we contribute to the state-of-the-art of this literature by
devising algorithms with provable guarantees  for the maximization of two natural inter-group  criteria, namely the minimum spacing and the minimum spanning tree spacing. The former is the minimum distance between points in different groups while the  latter captures separability through the cost of the minimum spanning tree that connects all groups. We obtain results for both the unrestricted case, in which no constraint on the clusters is
imposed, and for the constrained case where each group is required to have a minimum number of points. Our constraint is motivated
by the fact that the popular \singlelink, which optimizes both criteria in the unrestricted case, produces clusterings with many   tiny groups.

To complement our work, we present an empirical study with 10 real datasets, providing evidence that our methods work very well in practical settings. 

\end{abstract}

\section{Introduction}

Data clustering is a fundamental tool in machine learning that is commonly used in exploratory analysis and to reduce the computational resources required to handle large datasets. For comprehensive descriptions of different clustering methods and their applications, we refer to \cite{Jain:1999,HMMR15}. In general, clustering is the problem of partitioning a set of items so that similar items are grouped together and dissimilar items are separated. Internal measures that are used to assess the quality of a clustering (e.g. Silhouette coefficient \cite{Rousseeuw87silhouetteCluster} and Davies-Bouldin index \cite{davies_cluster_1979}) usually take into account  intra-group and inter-group criteria. The former considers the cohesion of a group while the latter measures how separated the groups are.

There are many papers in the literature that propose algorithms with provable approximation guarantees for optimizing  intra-group criteria. Algorithms   for $k$-center, $k$-medians and $k$-means cost functions \cite{DBLP:journals/tcs/Gonzalez85,DBLP:journals/jcss/CharikarGTS02,DBLP:journals/siamcomp/AhmadianNSW20} are some examples. However, optimizing inter-group criteria is much less understood. Here, we contribute to the state-of-the-art  by considering  the maximization of two natural inter-group  criteria, namely the minimum spacing and the minimum spanning tree spacing.

\medskip
\noindent {\bf Our Results.}
The spacing between two groups of points, formalized in Section \ref{sec:prelim},  is the minimum distance between a point in the first and a point in the second group. We consider two criteria for capturing the inter-group distance of a clustering, the  minimum spacing (\minsp) and minimum spanning tree spacing (\spmst). The former is given by the spacing of the two groups with the smallest spacing in the clustering. The latter, as the name suggests, measures the separability of a clustering according to the cost of the minimum spanning tree (MST) that connects its groups; the largest the cost, the most separated the clustering is. 

We first show that \singlelink, a procedure for building hierarchical clustering, produces a clustering that maximizes the MST spacing. This result contributes to a better  understanding of this very popular algorithm since the only provable guarantee that we are aware  of (in terms of optimizing some natural cost function) is that it  maximizes the minimum spacing of a clustering \cite{DBLP:books/daglib/0015106}[Chap 4.7]. Our guarantee [Theorem \ref{thm:mst-stronger}] is stronger than the existing one in the sense that any  clustering that maximizes the  MST spacing also  maximizes the minimum spacing. Figure \ref{fig:clustering} shows an example where the minimum spacing criterion does not characterize well the behavior of \singlelink.

Despite its nice properties, \singlelink \, tends to build clusterings with very small groups, which may be undesirable in practice. This can be seen in our experiments (see Figure \ref{fig:sigle-link-at-most-2-groups}) and is related to the well-documented fact that \singlelink \, suffers from the chaining effect \citep{Jain:1999}[Chap 3.2].

To mitigate this problem we consider the optimization of the aforementioned criteria under size constraints. Let $L$ be a given positive integer that determines the minimum size that every group in a clustering should have. A $(k, L)-$clustering is a clustering with $k$ groups in which the smallest group has at least $L$ elements. For the \minsp \, criterion we devise an algorithm that builds clustering with at least $L(1-\epsilon)$  points per group while guaranteeing that its minimum spacing is not smaller than that of an optimal $(k, L)$-clustering.  This result is the best possible in the sense that, unless $P=NP$, it is not possible to obtain an approximation with subpolynomial factor for the case in which the clustering is required to satisfy the hard constraint of $L$ points per group. For the \spmst \, criterion we also devise an algorithm with provable guarantees. It produces a clustering whose \spmst \, is at most a $\log k $ factor from the optimal one and whose groups have each at least $\rho L(1-\epsilon)/2$ points, where $\rho$ is a number in the interval $[1,2]$ that depends
on the ratio $n/kL$. We also prove that the maximization of this criterion is APX-Hard for any fixed $k$.

\begin{figure}[ht]
\vskip 0.2in
\begin{center}
\centerline{\includegraphics[width=\columnwidth]{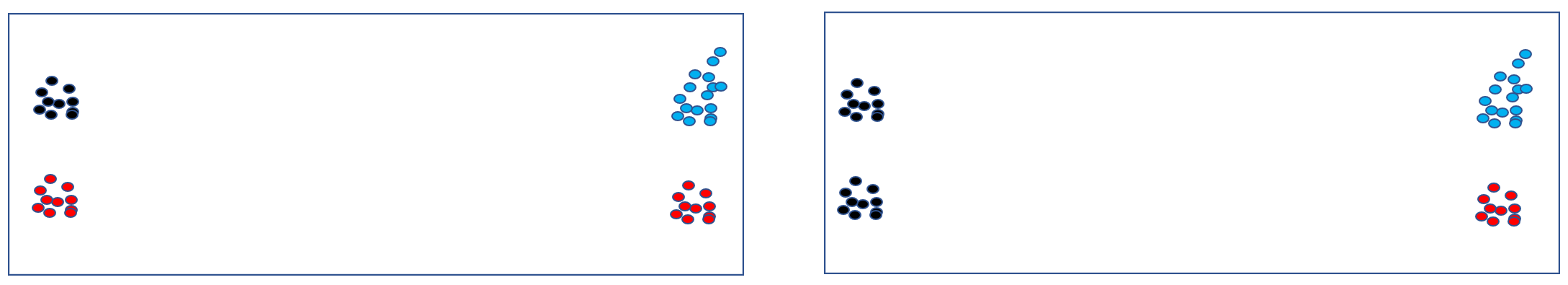}}
\caption{Partitions with 3 groups (defined by colors) that
maximize the minimum spacing. The rightmost one is built by \singlelink, but both of them maximize the minimum spacing -- showing that this condition alone is insufficient to properly characterize \singlelink's behavior.}
\label{fig:clustering}
\end{center}
\vskip -0.2in
\end{figure}

We complement our investigation  with an experimental study where we compare the behaviour of the clustering produced by our proposed algorithms  with those produced by $k$-means and  \singlelink \, for 10 real datasets. Our algorithms, as expected, present better results than $k$-means for the criteria under consideration while avoiding the issue of small groups observed for \singlelink. 

\medskip
\noindent {\bf Related Work.}
We are  only aware of a few works that  propose clustering algorithms with provable guarantees for optimizing separability (inter-group) criteria. We explain some of them.

The maximum $k$-cut problem is a widely studied problem in the combinatorial optimization community and  its solution can be naturally viewed as a clustering that optimizes a separability criterion. Given an edge-weighted  graph $G$ and an integer $k \ge 2$, the  maximum $k$-cut problem consists of finding a $k$-cut (partition of the vertices of the graph into $k$ groups) with maximum weight, where the weight of a cut is given by the sum of the weights of the edges that have vertices in different groups of the cut. The weight of the $k$-cut can be seen as a separability criterion, where the distance between two groups is given by the sum of the  pairwise distances of its points. It is well-known that a random assignment of the vertices (points) yields an $(1-1/k)$-approximation algorithm. This bound can be slightly improved using semi-definitive programming \cite{DBLP:journals/algorithmica/FriezeJ97}.

The minimum spacing, one of the criteria we studied here, also admits an algorithm with provable guarantees. In fact, as we already mentioned, it can be maximized in polynomial time via the \singlelink \, algorithm  \cite{DBLP:books/daglib/0015106}[Chap 4.7]. In what follows, we discuss works that study this algorithm as well as works that study problems and/or methods related to it.

\singlelink \, has been the subject of a number of researches \citep{zahn1971graph,kleinbergimpossibility,DBLP:journals/jmlr/CarlssonM10,hofmeyr2020connecting}. \cite{kleinbergimpossibility} presents an axiomatic study of clustering, the main result of which is a proof that it is not possible to define clustering functions that simultaneously satisfy three desirable properties introduced in the paper. However, it was shown that by choosing a proper stopping rule for the \singlelink \, it  satisfies any two of the three properties. \cite{DBLP:journals/jmlr/CarlssonM10} replaces the property of \cite{kleinbergimpossibility} that the clustering function must output a partition with the property that it must generate a tree (dendogram). Then, it was established that  \singlelink \, satisfies the new set of properties. In more recent work, \cite{hofmeyr2020connecting} establishes a connection between minimum spacing and spectral clustering. While the aforementioned works prove that \singlelink \, has important properties, in practice, it is reported that sometimes it presents poor performance due to the so-called chaining effect  (\cite{Jain:1999}[Chap 3.2]).

\singlelink \, belongs to the family of algorithms that are used to build Hierarchical Agglomerative Clustering (HAC). \cite{DBLP:conf/stoc/Dasgupta16} frames the problem of building a hierarchical clustering as a combinatorial optimization problem, where the goal is to output a hierarchy/tree that minimizes a given cost function. The paper proposes a cost function that has desirable properties and an algorithm to optimize it. This result was improved and extended by a series of papers \citep{DBLP:conf/nips/RoyP16,DBLP:conf/soda/CharikarC17,DBLP:journals/jacm/Cohen-AddadKMM19}. The algorithms discussed in these papers do not seem to be employed in practice. Recently, there has been some effort to analyse  HAC algorithms  that are popular in practice, such  as  the Average Link \citep{DBLP:conf/nips/MoseleyW17,DBLP:journals/jacm/Cohen-AddadKMM19}. Our investigation of \singlelink \, can be naturally connected with this line of research.

Finally, in a recent work,  \cite{ahmadi22individual} studied the notion of individual preference stability (IP-stability) in which a point is IP-stable if it is closer to its group (on average) than to any other group (on average). The clustering produced by \singlelink \, presents a kind of  individual preference stability in the sense that each point is closer to some point in its group than to any point in some other group. 

\noindent {\bf Potential Applications.}
We discuss two cases in which the maximization of inter-group criteria via our algorithms may be relevant: ensuring data diversity when training machine learning algorithms and population diversity in candidate solutions for genetic algorithms.

When training a machine learning model, ensuring data diversity may be crucial for achieving good results \citep{gong2019diversity}. In situations in which all available data cannot be used for training (e.g. training in the cloud with budget constraints), it is important to have a method for selecting a diverse subset of the data, and our algorithms can be used for this: to select $n = kL$ elements, one can partition the full data set into $k$  clusters, all of them containing at least $L$ members, and then select $L$ elements from each cluster. Using \singlelink \, to create $kL$ groups and then picking up one element per group is also a possibility but it would increase the probability of an over-representation of outliers in the obtained subset, as these outliers would likely be clustered as singletons (see Figure \ref{fig:sigle-link-at-most-2-groups} in Section \ref{sec:experiments}).

Note that our algorithms can be used to create not only one but several diverse and disjoint subsets, which might be relevant to generate partitions for cross-validation or for evaluating a model's robustness. For that, each subset is obtained by picking exactly one point per group.

For genetic algorithms, maintaining diversity over the iterations is important to ensure a good exploration of the search space \citep{gupta2012overview}. If all candidate solutions become too similar, the algorithm will become too dependent on mutations for improvement, as the offspring of two solutions will likely be similar to its two parents; and mutation alone may not be enough to fully explore the search space. We can apply our algorithms in a similar manner as mentioned above, by partitioning, at each iteration, the solutions into $k$ clusters of minimum size $L$ and selecting the best $L$ solutions from each cluster, according to the objective function of the underlying optimization problem, to maintain the solution population simultaneously optimized and diverse. Using \singlelink \, could lead to several poor solutions being selected to remain in the population, in case they are clustered as singletons.

We finally note that an algorithm that optimizes intra-group measures (e.g. $k$-means or $k$-medians) would not necessarily guarantee diversity for the aforementioned applications, as points from different groups can be close to each other.

\section{Preliminaries}
\label{sec:prelim}
Let ${\cal X}$ be a set of $n$ points and let $\dist:{\cal X} \times {\cal X} \mapsto \mathbb{R}^+$ be a distance function that maps every pair of points in ${\cal X}$ into a non-negative real.

Given a $k$-clustering ${\cal C}=(C_1,\ldots,C_k)$, we define the spacing    between two distinct groups $C_i$ and $C_j$ of ${\cal C}$  as $$\spac(C_i,C_j):=\min_{x \in C_i, y \in C_j} \{\dist(x,y)\}.$$ Then, the minimum spacing of ${\cal C}$ 
is given by $$\mbox{\minsp} ( {\cal C} ):= \min_{C_i, C_j \in {\cal C} \atop i \ne j }  \{ \spac(C_i,C_j) \}.$$
A $k$-clustering ${\cal C}$ induces on $(\X,\dist)$ an edge-weighted  complete graph $G_{\cal C}$
whose vertices are the groups of ${\cal C}$ 
and the weight $w(e_{ij})$  of the edge $e_{ij}$ between 
groups $C_i$ and $C_j$ is given by $w(e_{ij})=\spac(C_i,C_j)$.
For a set of edges $S \in G_{\cal C}$ we define $w(S):=\sum_{e \in S} w(e)$.

The {\em minimum  spanning tree spacing} of ${\cal C}$ (\spmst(${\cal C}$)) is  defined as the sum of the weights of the edges of the  minimum  spanning tree (MST$(G_{\C}))$ for  $G_{\cal C}$. In formulae,
$$ \mbox{\spmst}({\cal C}):= \sum_{e \in MST(G_{\cal C})} w(e).$$
Here, we will be  interested in the problems of finding partitions with maximum  \minsp\, and maximum \spmst \, both in the unrestricted case in which no constraint on the groups is imposed and  in the constrained case where each group is required to have at least $L$ points.

\medskip

\noindent {\bf Single-Linkage}.
We briefly explain \singlelink. The algorithm starts with $n$ groups, each of them consisting of a point in ${\cal X}$. Then, at each iteration,  it merges the two groups with minimum spacing into a new group. Thus, by the end of the iteration $n-k$ it obtains a clustering with $k$ groups. In \cite{DBLP:books/daglib/0015106} it is proved that the \singlelink \, obtains a $k$-clustering with maximum minimum spacing.

\begin{theorem}[\cite{DBLP:books/daglib/0015106}, chap 4.7]
 \label{ref:thm-spacing}
 The \singlelink \, algorithm obtains the $k$-clustering
 with maximum \minsp \, for instance $(\X,\dist)$.
\end{theorem}

\singlelink \, and minimum spanning trees are closely related since the former can be seen as the Kruskal's algorithm for building MST's with an early stopping rule. To analyze our algorithms we make use of well-known properties of MST's as  
the cut property[\cite{DBLP:books/daglib/0015106}, Property 4.17]
and the cycle property[\cite{DBLP:books/daglib/0015106}, Property 4.20]. Their statements can be found in the appendix.

For ease of presentation, we assume that all values of $\dist$ are distinct. We note, however, that our results hold if this assumption is dropped.

\section{Relating  \minsp \, and  \spmst \, criteria}
\label{sec:theoretical-algorithms}

We show that \singlelink \, finds the clustering with maximum  \spmst \, and  the maximization of \spmst \, implies the maximization of \minsp. These results are a consequence of Lemma \ref{lemm:aux-single-link} that generalizes the result of Theorem \ref{ref:thm-spacing}. The proof of this lemma can be found in the appendix.

Fix an instance $I=(\X,\dist,k)$. In what follows, $\C_{SL}$ is a $k$-clustering obtained by \singlelink \, for instance $I$ and $T_{SL}$ is a MST for $G_{\C_{SL}}$. Moreover, $w^{SL}_i$ is the weight of the $i$-th smallest weight of $T_{SL}$.
 
\begin{lemma} 
 Let $\C$ be a $k$-clustering for $I$ and let
 $w_i$ be the weight of the $i$-th smallest
 weight in a MST $T$ for the graph $G_{\C}$.
  Then, $w^{SL}_i \ge w_i$.
  \label{lemm:aux-single-link}
\end{lemma}

\begin{theorem}
\label{teo:singlelink-mst}
The clustering $\C_{SL}$ returned by  \singlelink  \, for instance $(\X,\dist,k)$ maximizes the \spmst \,  criterion. 
\end{theorem}
\begin{proof}
Let $\C$ be a $k$-clustering for $(\X,\dist,k)$
and let $w_i$ be the weight of the $i$-th cheapest edge of the MST for $G_{\C}$. Since $w^{SL}_i \ge w_i$ for $i=1,\ldots,k-1$,
we have that  $$\mbox{\spmst}(\C_{SL})= \sum_{i=1}^{k-1} w^{SL}_i \ge \sum_{i=1}^{k-1} w_i = \mbox{\spmst} (\C).$$
\end{proof}

\begin{theorem} 
\label{thm:mst-stronger}
Let $\C^*$ be a clustering that maximizes
the \spmst \, criterion for instance
$(\X,\dist,k)$. Then, it also maximizes \minsp \, for this same instance.
\end{theorem}
\begin{proof}
Let us assume that   $\C^*$ 
maximizes the  \spmst \, criterion but it does not maximize the \minsp \, criterion.
Thus, $w^{SL}_1 > w^*_1$, where $w^*_1$ is the minimum spacing of $\C^*$. It follows from the previous lemma that
$$\mbox{\spmst}(\C_{SL})= \sum_{i=1}^{k-1} w^{SL}_i > \sum_{i=1}^{k-1} w^*_i = \mbox{\spmst} (\C^*),$$
which contradicts the assumption that $\C^*$ maximizes the \spmst \, criterion.
\end{proof}

The next example (in the spirit of Figure \ref{fig:clustering}) shows that 
a partition that maximizes 
the \minsp \, criterion may have a poor result in terms of the \spmst \, criterion.

\begin{example}
\label{example1}
Let $D$ be a positive number much larger than $k$. Moreover,  let
$[t]$ be the set of the $t$ first positive integers and
$S=\{(D\cdot i,j)| i,j   \in [k-1]\} \cup (D,k) $ 
be a set of $(k-1)^2+1$ points in $\mathbb{R}^2$.

\singlelink \, builds a $k$-clustering with \minsp \, 1 and  \spmst $1+(k-2)D$ for $S$. 

However, the $k$-clustering $(C_1,\ldots,C_k)$,
where $C_j = \{(D \cdot i,j ) | i=1,\ldots,k-1 \}$, for $j<k$
and $C_k=\{(D,k)\}$ has \minsp \, $1$ and \spmst$=(k-1)$,  

\end{example}

\section{Avoiding small groups}
\label{sec:small-groups}
In this section, we optimize our criteria
under the constraint that
all groups must have at least $L$ points, where
$L$ is a positive integer not larger than $n/k$ that is provided  by the user. Note that the problem is not feasible if $L>n/k$.

We say that an algorithm has $(\beta,\gamma)$-approximation for a criterion $\kappa \in \{\minsp, \spmst\}$ if
for all instances $I$ it obtains a clustering
$\C_I$ such that $\C_I$ is a $(k, \lfloor \beta L \rfloor )$-clustering and the value of $\kappa$ for 
$\C_I$ is at least
$  \gamma \cdot OPT$, where $OPT$ is the maximum possible value of
 $\kappa$ that can be achieved for a $(k,L)$-clustering.
 
We first show  how to obtain a $(1-\epsilon,1) $
for   the \minsp \, criterion.

\subsection{The \minsp \, criterion}

We start with the polynomial-time approximation scheme for the \minsp \, criterion.
Our method uses, as a subroutine, 
an algorithm for the max-min scheduling 
problem with identical machines
\cite{DBLP:journals/orl/CsirikKW92,DBLP:journals/orl/Woeginger97}.
Given $m$ machines and a set of  $n$ jobs, with 
processing times $p_1,\ldots,p_n$,
the problem consists of finding an assignment
of jobs to the machines so that the load
of the machine with minimum load is maximized.
This problem admits a polynomial-time approximation scheme
\cite{DBLP:journals/orl/Woeginger97}.

Let {\tt MaxMinSched}($P,k,\epsilon$) 
be a routine that implements this scheme. It receives as input
a
parameter $\epsilon>0$, an integer $k$ and a list  of numbers $P$ (corresponding to processing times). Then, it  returns a partition of $P$
 into $k$ lists (corresponding to machines) such that
the sum of the numbers of the list with minimum sum 
is at least $(1-\epsilon)OPT$, where $OPT$
is the minimum load of a machine in an  optimal solution of 
for the max-min scheduling 
when the list of processing
times is $P$ and the number of machines is $k$.

Algorithm \ref{alg:constrained-mispacing},
as proved in the next theorem, obtains a $(k,L(1-\epsilon))$-clustering
whose \minsp \, is at least the \minsp \,
of an optimal $(k,L)$-clustering.
For that, it looks for the largest integer $t$
for which the clustering  $\A_t$
obtained by executing $t$ steps of \singlelink \, and then
combining the resulting groups into $k$ groups (via {\tt MaxMinSched})  is a  $(k,L(1-\epsilon))$-clustering.
We assume that {\tt MaxMinSched}, in addition of
returning the partition of the sizes,
also returns the group associated to each
size.

\begin{algorithm}[H]
  \caption{{\tt AlgoMinSp}($\X$; $\dist$; $k$; $\epsilon>0$;$L$)}
  
  \begin{algorithmic}[]

\STATE $t \leftarrow n-k$
\WHILE{$t \ge 0$ }
\STATE Run $t$ merging steps of the \singlelink \, for
input $\X$

\STATE Let $C_1,\ldots,C_{n-t}$ the groups obtained
by the end of the  $t$ steps

\STATE  $P \gets (|C_1|,\ldots,|C_{n-t}|)$

\STATE $A_{t} \gets $
{\tt MaxMinSched}($P,k,\epsilon$)

\IF {the smallest group in $A_t$ has size
greater than or equal to $L(1-\epsilon)$}
\STATE Return $A_{t}$
\ELSE
\STATE $t \gets t-1$

\ENDIF
\ENDWHILE

  \end{algorithmic}
  \label{alg:constrained-mispacing}
\end{algorithm}

\begin{theorem}
Fix $\epsilon>0$. The clustering $\A_t$ returned by the Algorithm
\ref{alg:constrained-mispacing} is a $(k,(1-\epsilon)L)$-clustering that
satisfies $ \minsp (\A_t) \ge \minsp(\C^*)$,
where $\C^*$ is the $(k,L)$-clustering
with maximum \minsp.
\label{thm:ContrainedMinSP}
\end{theorem}
\begin{proof}
By design, $\A_t$ has $k$ groups with at least $L(1-\epsilon)$ points in each of them.

For the sake of contradiction, 
let us assume that $\minsp(\A_t) < \minsp(\C^*)$.
 
Let   $\C=(C_1,\ldots,C_{n-t})$ be the
 list of $n-t$ groups
 obtained when $t$ merging steps of \singlelink \, are 
performed.
We assume w.l.o.g. that
 $C_1$ and $C_2$
are the two  groups with a minimum spacing in this list, so
that $\minsp(\C) = {\tt spacing}(C_1,C_2)$.
Since $\A_t$ is a $k$-clustering that is obtained
by merging groups in $\C$ we have
  $\minsp(\A_t) \ge \minsp(\C) = {\tt spacing}(C_1,C_2)$.
 
 For $i=1,\ldots, n-t$  we have that
 $C_i \subseteq C$ for some $C \in \C^*$,
 otherwise we would have  $\minsp(\A_t) \ge \minsp(\C^*)$. In addition,
 we must have $C_1 \cup C_2 \subseteq C$ 
 for some $C \in \C^*$, 
 otherwise, again, we would have  $\minsp(\A_t) \ge \minsp(\C^*)$.

We can conclude that there is a feasible solution with 
minimum load not smaller than $L$ for
the max-min scheduling problem with processing times
$P'=(|C_1 \cup C_2|,|C_3|,\ldots,|C_{n-t}|)$ and $K$ machines.
Thus, by running 
 $t+1$ steps of  \singlelink \, followed by 
 {\tt MaxMinSched}$(P',k,\epsilon)$,  we  
would  get
a $k$-clustering whose smallest group
has at least $L(1-\epsilon)$ points.
This implies that the  algorithm  would
have stopped after performing $t+1$
merging steps, which is a contradiction.
 \end{proof}

Algorithm \ref{alg:constrained-mispacing}, as presented, may run  
\singlelink \, $n-k$ times, which may be  quite expensive.
Fortunately, it  admits an implementation
that runs \singlelink \, just once, and performs an inexpensive binary search to find a suitable $t$. 

The next theorem shows that Algorithm \ref{alg:constrained-mispacing}
has essentially tight guarantees under the hypothesis that
$P \ne NP$. The proof can be found in the appendix

\begin{theorem} Unless $P=NP$, for any $\alpha =poly(n)$, the problem of 
finding the $(k,L)$-clustering
that maximizes the \minsp \, criterion does not
admit a $(1,\frac{1}{\alpha})$-approximation.
\label{thm:complexity}
 \end{theorem}

\subsection{The \spmst \, criterion}

Now, we turn to the   \spmst \, criterion.
Let $\rho := \min\{\frac{n/k}{L} ,2 \}$.
Our main  contribution is Algorithm \ref{alg:constrained-MST-SP-2},
it obtains a   $(\frac{\rho(1-\epsilon)}{2}, \frac{1}{H_{k-1}})$ approximation for this criterion, where
 $H_{k-1}=\sum_{i=1}^{k-1} \frac{1}{i}$ is the $(k-1)$-th Harmonic number. Note that $H_{k-1}$ is $\Theta(\log k)$.

In high level, for each $\ell=2,\ldots,k$,
the algorithm  calls {\tt AlgoMinSp} (Algorithm  \ref{alg:constrained-mispacing}) to build a clustering 
${\cal A}'_{\ell}$
with $\ell$ groups and 
then it transforms ${\cal A}'_{\ell}$  (lines
5-13)  into a clustering
${\cal A}_{\ell}$ with $k$ groups. 
In the end, it returns
the clustering, among the $k-1$ considered, with maximum \spmst.

\begin{algorithm}[H]
  \caption{{\tt Constrained-MaxMST}($\X$; $\dist$; $k$; $L$; $\epsilon$)}
  
  \begin{algorithmic}[1]

 \FOR{ $\ell=2,\ldots,k$}
\STATE ${\cal A}'_{\ell} \gets$ {\tt AlgoMinSp}($\X$,$\dist$,$\ell$,$\epsilon$,$L$)
\STATE  {\tt NonVisited} $\gets {\cal A}'_{\ell}$
\STATE ${\cal A}_{\ell} \gets \emptyset$
\FOR{each $A'$ in ${\cal A}'_{\ell}$, iterating from the largest group to the smallest} \label{line-innerfor} 
\STATE {\tt NonVisited} $\gets$ {\tt NonVisited} - $A'$ 
\medskip
\STATE  {\tt SplitNumber} $\gets \left \lfloor \frac{2|A'|}{\rho (1-\epsilon)L} \right \rfloor $ 
\IF{ $|\A_{\ell}|$+ | {\tt NonVisited} | 
+  {\tt SplitNumber}
$< k$}

\STATE Split $A'$ into  {\tt SplitNumber} as balanced as possible groups and add them to ${\cal A}_{\ell}$

\ELSE

\STATE  Split $A'$ into $k-|\A_{\ell}| -$ | {\tt NonVisited} | \, as balanced as possible groups; add them to ${\cal A}_{\ell}$
\STATE Add all groups in {\tt NonVisited} to ${\cal A}_{\ell}$
\STATE Break

\ENDIF
\ENDFOR

\ENDFOR
\STATE {\bf Return} the clustering $A_{\ell}$, among the $k-1$ obtained, that has the maximum
\spmst
\end{algorithmic}
  \label{alg:constrained-MST-SP-2}
\end{algorithm}

We remark that 
we do not  need to scan the groups in $\A'_{\ell}$ by non-increasing order of their sizes
to establish our guarantees presented below. However, this rule tends to avoid groups with sizes smaller than $L$.

\begin{lemma}
Fix $\epsilon>0$. Thus, for each $\ell$, every group in $\A_{\ell}$ has
at least
$\left \lfloor \frac{ \rho (1-\epsilon) L}{ 2} \right \rfloor$ points.
\label{lem:LRestriction}
\end{lemma}
\begin{proof}
 The groups that are  added to ${\cal A}_\ell$ in line 12 have at least $(1-\epsilon)L$ points while the number of points of those that are added at  either line 11 or 9 is at least
$$\left \lfloor \frac{|A'|}{ \lfloor 2 |A'|   /  \rho (1-\epsilon) L \rfloor}  \right \rfloor \ge 
\left \lfloor \frac{ \rho (1-\epsilon) L}{ 2} \right \rfloor$$   
Moreover,  if the {\bf For} is
not interrupted by the {\bf Break} command, 
the total number of groups in ${\cal A}_\ell$
is 
$$\sum_{ A' \in {\cal A}'_{\ell}} \left \lfloor \frac{2 |A'|}{  \rho (1-\epsilon) L} \right \rfloor 
\geq \sum_{ A' \in {\cal A}'_{\ell}} 
 \frac{2 |A'|}{  \rho (1-\epsilon)L} - \ell  \ge \frac{2n }{\rho (1-\epsilon) L} -k  
\ge \frac{2 k }{(1-\epsilon)} -k 
\ge k$$
Since the {\bf For} is interrupted as soon as $k$
groups can be obtained then,  ${\cal A}_\ell$ has
$k$ groups.
\end{proof}

For the next results, we use  $\C^*$ to denote the $(k,L)$-clustering
with maximum \spmst \, and  $w^*_i$ to denote the cost of the 
$i$-th cheapest edge in the MST for  $G_{\C^*}$.
Our first lemma can be seen as a generalization of Theorem \ref{thm:ContrainedMinSP}. Its proof can be found in the appendix.

\begin{lemma}
For each 
$\ell$, \minsp (${\cal A}'_{\ell}$) $\ge w^*_{k- \ell +1} $. 
\label{lem:uppbound09May}
\end{lemma}

A simple consequence of the previous lemma is that 
the  \spmst\, of  clustering ${\cal A}'_\ell$ is at least
 $ (\ell-1) \cdot  w^*_{k-\ell+1}.$ The next lemma
 shows that this bound also holds for 
 the clustering ${\cal A}_\ell$. The proof
 consists of showing that each edge of a MST for
 ${\cal A}'_\ell$  is also an edge of a MST for ${\cal A}_\ell$.

\begin{lemma} For each $\ell=2,\ldots,k$ we have
 \spmst(${\cal A}_\ell$) $\ge (\ell-1) \cdot 
w^*_{k-\ell+1}.$
\label{lem:lowbound-15may} 
\end{lemma}
\begin{proof}
Let $T_\ell$ and $T'_\ell$ be, respectively,
the MST for $G_{{\cal A}_\ell}$ and $G_{{\cal A}'_\ell}$.  
By the previous lemma, each of the $(\ell-1)$ edges of $T'_\ell$  has cost
at least  $\minsp (\A'_{\ell}) \ge w^*_{k-\ell+1}$.
Thus, to establish the result, 
it is enough to argue that each edge of 
 $T'_\ell$ also  belongs to 
$T_\ell$.

We say that a group $A \in {\cal A}_\ell$ is
is generated from a group $A' \in {\cal A}'_\ell$
if $A=A'$ or 
$A$ is one of the balanced groups that is generated when $A'$ is split in
the internal {\bf For} of Algorithm \ref{alg:constrained-MST-SP-2}.  We say that a  vertex
$x$ in $G_{{\cal A}_\ell}$ is generated from a vertex $x'$ in $G_{{\cal A}'_\ell}$
if the group  corresponding to $x$  is generated
by the  corresponding to $x'$.

Let $e'=u'v'$ be an edge in
$T'_\ell$ and let $S'$ be
a cut in graph  $G_{{\cal A}'_\ell}$ whose
vertices are those from the  connected component of $T'_\ell \setminus e'$ that includes $u'$. 
We define the cut  $S$ of $G_{{\cal A}_\ell}$
as follows
$S=\{ x \in G_{{\cal A}_\ell} | x \mbox{ is generated 
from some } x' \in S'\}$. 

Let $u$ and $v$ be  vertices generated 
from $u'$ and $v'$, respectively, that satisfy
$w(uv)=w(u'v')$. It is enough to  show that $uv$ is the cheapest edge that crosses  $S$ since by the cut property  [\cite{DBLP:books/daglib/0015106}, Property 4.17.]
 this implies  that $uv \in T_\ell$.
We prove it by contradiction. Let us assume that there
is another edge $f=yz$ that crosses  $S$ and
has weight smaller than $w(uv)$.
Let $y'$ and $z'$ be vertices in $G_{{\cal A}'_\ell}$
that generate $y$ and $z$, respectively, and let
$f'=y'z'$. Thus, $w(f') \le w(f) < w(uv)=w(u'v')=w(e')$.
However, this contradicts 
the cycle property of MST's [\cite{DBLP:books/daglib/0015106}, Property 4.20] because it implies that
the edge with the largest weight in the cycle of
 $G_{{\cal A}'_\ell}$ comprised
by edge $f'$ and the path in $T'_\ell$ the connects
$y'$ to $z'$ belongs to the  $T'_\ell$.
\end{proof}

The next theorem is the main result of this section.

\begin{theorem}
Fix $\epsilon>0$. Algorithm
\ref{alg:constrained-MST-SP-2} is a $(\frac{(1-\epsilon)\rho}{2}, \frac{1}{H_{k-1}})$-approximation 
for the problem of finding the $(k,L)$--clustering that maximizes the \spmst \, criterion.
\label{thm:bound-contrained-mstsp}
\end{theorem} 
\begin{proof}
Let $\C$  be the clustering  returned by Algorithm
\ref{alg:constrained-MST-SP-2}.
Lemma \ref{lem:LRestriction} guarantees that 
 $\C$  is a $(k,\lfloor \frac{(1-\epsilon) \rho L}{2} \rfloor)$-clustering.

Thus, we just need to argue about $\spmst(\C)$.
We have that 
$$ \spmst(\C^*)=\sum_{i=2}^k w^*_{k-i+1}$$ and,
due to Lemma \ref{lem:lowbound-15may},
$ \spmst(\C) \ge \max \{ (\ell-1) \cdot 
w^*_{k-\ell+1} | 2 \le \ell \le k\}.$

Let $\tilde{\ell}$ be the value $\ell$ that maximizes
$(\ell-1) \cdot w^*_{k-\ell+1}$.
It follows that 
$ w^*_{k-i+1} \le ((\tilde{\ell}-1)/(i-1))  w^*_{k-\tilde{\ell}+1} .$ for $i=2,\ldots,k$.
Thus,

$$ \frac{ \spmst(\C^*) }{ \spmst(\C) } \le
\frac{ \sum_{i=2}^k w^*_{k-i+1} }{( \tilde{\ell}-1) \cdot w^*_{k-\tilde{\ell} +1}}
\le \frac{  (\tilde{\ell}-1) \cdot w^*_{k-\tilde{\ell}+1} 
\cdot \sum_{i=2}^k \frac{1}{i-1}  }{ (\tilde{\ell}-1) \cdot w^*_{k-\tilde{\ell} +1}} =H_{k-1}$$

\end{proof}

We end this section by showing that 
the optimization of \spmst \, is APX-HARD (for fixed $k$)
when a hard constraint on the number of points
per group is imposed.
The proof can be found in the appendix.

\begin{theorem} Unless $P=NP$, for any $\alpha =poly(n)$,
there is no (1,$\frac{k-2}{k-1}+\frac{1}{\alpha(k-1)}$)-approximation for the problem of  finding the $(k,L)$-clustering
that maximizes the \spmst \, criterion.
\label{thm:complexity2}
 \end{theorem}

\section{Experiments}
\label{sec:experiments}

To evaluate the performance of Algorithms 
\ref{alg:constrained-mispacing} and
\ref{alg:constrained-MST-SP-2}, we ran experiments with 10 different datasets, comparing the results with those of \singlelink \, and of the traditional $k$-means algorithm from \cite{lloyd1982least} with a ++ initialization \citep{arthur2006k}. 
For the implementation of  routine {\tt MaxMinSched}, employed by Algorithm \ref{alg:constrained-mispacing}, we used the Longest Processing Time rule.
This rule has the advantage of being fast while
guaranteeing a $3/4$ approximation for the max-min scheduling problem \citep{DBLP:journals/orl/CsirikKW92}.
The code for running the algorithms can be found at \url{https://github.com/lmurtinho/SizeConstrainedSpacing}.

Our first experiment investigates the size of
the groups produced by \singlelink \, for the 10 datasets, whose dimensions can be found in the first two columns
of Table \ref{tab:comparison-methods}.
Figure \ref{fig:sigle-link-at-most-2-groups} shows the proportion of singletons for each dataset with the growth of $k$.
For all datasets but {\tt Vowel} and {\tt Mice} the
  majority of groups are singletons, even
for small values of $k$.
This undesirable behavior motivates our constraint on the minimum size of a group.

  \begin{figure}[htb]
   \begin{center}
   \includegraphics[width=0.8\textwidth]{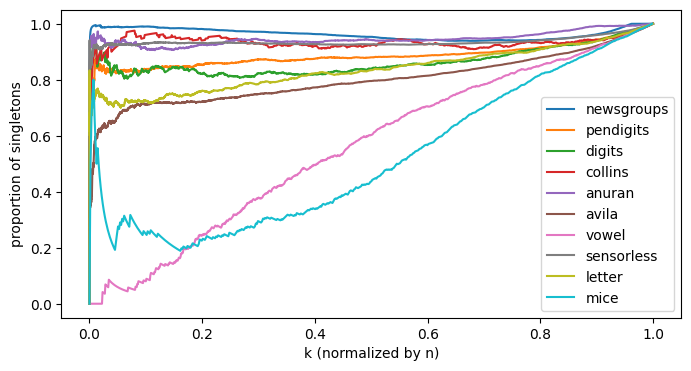}
   \end{center}
   \caption{Proportion of singletons for each dataset with the growth of $k$ }
   \label{fig:sigle-link-at-most-2-groups}
   \end{figure}

In our second experiment, we compare the values of \minsp \,  
and \spmst \, achieved by our algorithms with those
of $k$-means. While $k$-means is not a particularly
strong competitor in the sense that it was
not designed to optimize our criteria, the motivation to include it is that $k$-means is a very popular choice among
practitioners. Moreover, for datasets with well-separated groups, the minimization of the squared sum of errors (pursued by $k$-means) should also imply the maximization of inter-group criteria. 

 Table \ref{tab:comparison-methods} presents 
 the results of this experiment. 
 The  values chosen
for $k$ are the numbers of classes (dependent variable) in the datasets, while for the $\dist$ function  we employed the Euclidean distance.
 The values associated with the criteria 
  are averages of 10 executions, with each execution corresponding to a different seed provided to $k$-means.
To set the  value of $L$ for the $i$-th execution of our algorithms, we take the size of the smallest group generated by $k$-means for this execution and multiply it by $4/3$. This way, we guarantee that the size of the smallest group produced by our methods, for each execution,  is not smaller than that of $k$-means, which makes the comparison among the optimization criteria fairer.

\begin{table}[]
    \caption{\minsp \, and \spmst \, for the
    different methods and datasets.}
\label{tab:comparison-methods}

\begin{center}
\begin{tabular}{c|cc||ccc||ccc}
\toprule
\multirow{2}{*}{}  & \multicolumn{2}{c||}{Dimensions} & \multicolumn{3}{c||}{\minsp}      & \multicolumn{3}{c}{\spmst}     \\
     & n & k                & Algo \ref{alg:constrained-mispacing} & Algo \ref{alg:constrained-MST-SP-2} & k-means  & Algo \ref{alg:constrained-mispacing}  & Algo \ref{alg:constrained-MST-SP-2} & k-means \\ \midrule

anuran     & 7,195  & 10 & \textbf{0.19}  & 0.09           & 0.05  & 1.71            & \textbf{1.87}   & 1.01   \\
avila      & 20,867 & 12 & \textbf{0.07}  & 0.04           & 0     & 0.77            & \textbf{0.81}   & 0.66   \\
collins    & 1,000  & 30 & \textbf{0.42}  & \textbf{0.42}  & 0.22  & \textbf{12.42}  & \textbf{12.42}  & 8.58   \\
digits     & 1,797  & 10 & \textbf{19.74} & \textbf{19.74} & 13.79 & \textbf{178.22} & \textbf{178.22} & 145.13 \\
letter     & 20,000 & 26 & \textbf{0.2}   & 0.11           & 0.07  & 4.98            & \textbf{5.67}   & 1.98   \\
mice       & 552   & 8  & \textbf{0.79}  & \textbf{0.79}  & 0.24  & \textbf{5.66}   & \textbf{5.66}   & 2.37   \\
newsgroups & 18,846 & 20 & \textbf{1}     & 1              & 0.17  & 19              & \textbf{19}     & 8.4    \\
pendigits  & 10,992 & 10 & \textbf{23.89} & 9.08           & 8.31  & 215.11          & \textbf{217.01} & 119.85 \\
sensorless & 58,509 & 11 & \textbf{0.13}  & 0.08           & 0.03  & 1.31            & \textbf{1.36}   & 1.29   \\
vowel      & 990   & 11 & \textbf{0.49}  & \textbf{0.49}  & 0.11  & \textbf{4.94}   & \textbf{4.94}   & 1.84   \\

\bottomrule

\end{tabular}
\end{center}
\end{table}

With respect to the \minsp \, criterion,  Algorithm \ref{alg:constrained-mispacing} is at least as good
as Algorithm \ref{alg:constrained-MST-SP-2} for every dataset
(being superior on 6) and both 
Algorithm \ref{alg:constrained-mispacing} and \ref{alg:constrained-MST-SP-2}   outperform $k$-means on all datasets.
On the other hand, with respect to the \spmst \, criterion,  Algorithm \ref{alg:constrained-MST-SP-2} is at least as good as  Algorithm \ref{alg:constrained-mispacing} for every dataset (being better on 6)  and, again, both algorithms outperform $k$-means for all datasets.  In the appendix,
we present additional information regarding this experiment.
In particular, we show that the \spmst \, achieved
by Algorithm \ref{alg:constrained-MST-SP-2}
is on average 81$\%$ of the upper bound 
$\sum_{\ell=2}^{k} \minsp(\A'_{\ell})$
that follows from Lemma \ref{lem:uppbound09May}.  
This is much better than the $1/H_{k-1}$ ratio given by the theoretical bound.

Finally, Table \ref{tab:running-times} shows the  running time of our algorithms for the datasets that consumed  more time. We observe that the overhead introduced by Algorithm \ref{alg:constrained-mispacing} w.r.t. \singlelink \, is negligible while Algorithm  \ref{alg:constrained-MST-SP-2}, as expected, is more costly. In the appendix, we show that a strategy that only considers values of $\ell$ that can be written as $ \lceil k / 2^t \rceil$, for  $t=0,\ldots,\lfloor \log k \rfloor$,    in the first loop of Algorithm \ref{alg:constrained-MST-SP-2} provides a significant gain of running time while incurring a small loss in the \spmst. We note that the $\log k$ bound of Theorem \ref{thm:bound-contrained-mstsp} is still valid for this strategy.

\begin{table}[]
\caption{Running time in seconds of \singlelink \, and our methods.}
\label{tab:running-times}
\begin{center}

    \begin{tabular}{cccc}
\toprule
Dataset & \singlelink & Algo 1 & Algo 2 \\
\midrule
sensorless       & 93.8       & 99.4      & 701.4      \\
newsgroups       & 274.2         & 276.4      & 440.4      \\
letter           & 4.6        & 5.9        & 116.4    \\
avila            & 4.0        & 5.5        & 62.5        \\
pendigits        & 1.4        & 2.1        & 16.0   \\
\bottomrule
\end{tabular}
\normalsize

\end{center}
\end{table}

\section{Final Remarks}
We have endeavored in this paper to expand the current knowledge on clustering methods for optimizing inter-cluster criteria. We have proved that the well-known \singlelink \,  produces partitions that maximize not only the minimum spacing between any two clusters, but also the MST spacing, a stronger guarantee. We have also studied the task of maximizing these
criteria under the constraint that
each group of the clustering has at least
$L$ points.
We provided complexity results
and algorithms with provable approximation guarantees.

One potential limitation of our proposed  algorithms
is their usage on massive datasets (in particular Algorithm 
\ref{alg:constrained-MST-SP-2}) since
they execute \singlelink \ one or many times. 
If the $\dist$ function is explicitly given, then the
$\Omega(n^2)$ time spent by \singlelink \, is unavoidable. However, if the distances
can be calculated from the set of points $\X$ then faster
algorithms might be obtained.

The main theoretical question that remained open in our work is whether
there exist constant approximation algorithms for the maximization of \spmst. In addition to addressing this question, interesting directions for future research include handling
different inter-group measures as well as other constraints on the
structure of clustering.

\begin{ack}
The authors thank BigDataCorp (\url{https://bigdatacorp.com.br/}) for providing computational power for the experiments of an initial version of the paper.

The work of the authors is partially supported by the Air Force Office of Scientific Research (award number FA9550-22-1-0475). 

The work of the first author is partially supported by 
 CNPq (grant 310741/2021-1).
\end{ack}

\bibliography{biblio.bib}

\bibliographystyle{icml2023}

\newpage
\appendix
\onecolumn

\section{Properties of Minimum Spanning Trees}
\label{app:mst-properties}

\begin{theorem}[Cut Property, \cite{DBLP:books/daglib/0015106}, Property 4.17.]
\label{thm:cut-property}
 Let $G=(V,E)$ be a graph with distinct weights on its edges.
 Let $S \subset V$ be a non-empty cut in $G$,
 If $e$ is the edge with minimum cost
 among those that have one endpoint in $S$ and
 the other one in $V \setminus S$, then
 $e$ belongs to every MST for $G$
 \end{theorem}

\begin{theorem}[Cycle Property, \cite{DBLP:books/daglib/0015106}, Property 4.20.]
\label{thm:cycle-property}
 Let $G=(V,E)$ be a graph with distinct weights on its edges
 and let $C$ be a cycle in $G$. Then,
 the edge with the largest
 weight in $C$ does not belong to any minimum spanning
 tree for $G$.
\end{theorem}

The following characterization of MST's will be useful.
Its correctness  follows directly from the cycle property.

\begin{theorem}
\label{thm:mst-coditions}
 Let $G=(V,E)$ be a graph with weights on its edges.
 A spanning tree $T$ for $G$ is a MST for $G$ if and only if for each edge $e=uv$ in $E$, the weight $w(e)$ of $e$ satisfies $w(e) \ge w(e')$ for every edge $e'$ in the path that connects $u$ to $v$ in $T$.
\end{theorem}

\section{Proof of  Lemma \ref{lemm:aux-single-link}}
\label{app:claim-proof}

We need the  following proposition.

\begin{proposition} Let $\C'$ be  a $k$-clustering  for instance $I$ and let $T'$ be a MST for 
$G_{\C'}$.
Moreover, let $C'_i$ and $C'_j$ be groups of
$\C'$ such that ${\mbox \spac}(C'_i,C'_j)={\mbox \minsp}(\C')$.
Then, the tree $T^{a}$ that results from  the contraction
of the nodes $C'_i$ and $C'_j$ in $T'$ is a MST
for $G_{\C^a}$, where $\C^a$ is the $(k-1)$-clustering    obtained from $\C'$ by  merging $C'_i$ and $C'_j$.
\end{proposition}
\begin{proof}
We show that $T^{a}$ satisfies the conditions of 
 Theorem \ref{thm:mst-coditions} when
$G=G_{\C^{a}}$. For that, we will use the fact 
that $T'$  satisfies the conditions of
 Theorem \ref{thm:mst-coditions} when
 $G=G_{\C'}$.

Let $x$ and $y$ be nodes of $T^{a}$.
For the sake of contradiction, we assume that
edge $xy$ does not satisfy the conditions of 
 Theorem \ref{thm:mst-coditions} when
$G=G_{\C^{a}}$ and $T=T^a$. Let $w^*$ be the weight of edge $xy$ and let $e$ be an edge in the path that connects $x$ to $y$ in $T^a$ such that $w(e) > w^*$.  
We have two cases: 
\medskip

Case 1) $x \neq C'_i \cup C'_j$ and
$y \neq C'_i \cup C'_j$. Then,  $e$ is also an edge in the path that connects
$x$ to $y$ in $T'$. This implies that 
$xy$  does not  satisfy the required conditions when $G=G_{\C'}$ and $T=T',$ which is a contradiction.

\medskip

Case 2) $x = C'_i \cup C'_j$ or
$y \neq C'_i \cup C'_j$. Let us assume w.l.o.g. that $x = C'_i \cup C'_j$. 
Let $w'_i$ and $w'_j$ be, respectively, the weights of the edges $(y,C'_i)$ and
$(y,C'_j)$ in $G_{\C'}$.
We have that  $w^*=\min \{w'_i,w'_j\}$.

Let us assume w.l.o.g. that $w^*=w'_i$.
Then, $e$ is also an edge in the path that connects
$y$ to $C'_i$ in $T'$.
Again, this implies that 
$xy$  does not  satisfy the required conditions when $G=G_{\C'}$ and $T=T',$ which is a contradiction.
\end{proof}

\medskip

\noindent {\it Proof of Lemma \ref{lemm:aux-single-link}}.

It follows from the previous proposition hat 
the tree $T^{i-1}$ that is obtained by  contracting the $i-1$ cheapest edges of $T$ is a MST for $G_{\C^{i-1}}$, where 
$\C^{i-1}$ is 
a clustering for instance $I$ that contains $(k-(i-1))$ groups. 
The cheapest edge  of $T^{i-1}$ is exactly the $i$-th cheapest edge of $T$. Thus,   $\mbox{\minsp}(\C^{i-1})=w_i$.

Similarly, $w^{SL}_i$ is  exactly the 
minimum spacing of a clustering, with $(k-(i-1)$ groups, that is obtained by \singlelink \, for instance $I$.
Thus, it follows from Theorem \ref{ref:thm-spacing}  that  
$w^{SL}_i \ge w_i$.

\section{Proofs of Section \ref{sec:small-groups}}

\subsection{Proof of  Lemma \ref{lem:uppbound09May}}
\begin{proof}
Let $T^*$ be the MST for $G_{\C^*}$.
If we remove the $\ell-1$ most expensive edges of $T^*$ we obtain a forest $F$ with $\ell$ connected components.
The clustering $\C^*_\ell$ comprised
by $\ell$ groups in which the $i$th group corresponds to 
the $i$th connected component of $F$
is a $(\ell,L)$-clustering and $\minsp(\C^*_\ell)=w^*_{k- \ell +1}$.

Let $OPT$ be the $\minsp$\, of the 
$(\ell,L)$-clustering with maximum $\minsp$.
Thus, by Theorem \ref{thm:ContrainedMinSP} $\minsp ({\cal A}'_{\ell}) \ge OPT \ge \minsp(\C^*_\ell)=w^*_{k- \ell +1}$
 \end{proof}

\subsection{Proof of  Theorem \ref{thm:complexity}}
\label{app:complexity-proof}

\begin{proof}
We make a reduction from the $(T/4, T/2)$-restricted 3-PARTITION problem. Given a
multiset $S = \{s_1,\ldots,s_{3m}\}$ of positive integers that satisfies
$\sum_i s_i=mT$, the 3-PARTITION  problem consists of deciding whether or not there exists a partition of $S$ into $m$ triples such that the
sum of the numbers in each one is equal to $T$. In the $(T/4, T/2)$-restricted 3-PARTITION problem,
 there is an additional requirement that each number of $S$ should be in the interval $(T/4, T/2)$. This
 problem is strongly NP-COMPLETE \cite{garey1979computers}

The instance $I = (X, k,L, \dist)$ for our clustering problem is built as follows: we set $L = T$,
$ k = m$; for $i = 1,\ldots, 3m$ let $\X_i$ be a set with $s_i$ points so that the distance between points in
 the same group $\X_i$ is 1 while the distance between points in different groups is $\alpha + 1$. We set
$\X = \X_1 \cup \ldots \cup \X_{3m}$. Note that we are employing a pseudo-polynomial reduction but this is fine
 since the $( T/4,T/2)$-restricted 3-PARTITION problem is strongly NP-COMPLETE.

Let us assume that there is an  $(1, 1/\alpha)$-approximation for our problem and let $\C$ be the clustering
 returned by this algorithm for instance $I$. We argue that the answer to the 3-PARTITION problem is ’YES’ if and only if 
 $\minsp(\C)=\alpha+1$.

 First, we show that if the answer is ’YES’, there is a $k$-clustering $\C^*$ for $I $ with $\minsp(\C^*)=\alpha+1$. In
fact, let $S_1,\ldots,S_m$ be a solution of the 3-PARTITION problem and let $\{s_{i_1} , s_{i_2} ,s_{i_3} \}$ the numbers
in $S_i$. Let $\C^*$ be a $k$-clustering where the $i$th group is comprised by all points in $\X_{i_1} \cup \X_{i_2} \cup \X_{i_3}$.
Clearly, each group has $L$ points and the \minsp \, of this clustering is $\alpha + 1$. Since our algorithm is a (1, $1/\alpha $)-approximation it returns a clustering 
$\C$ with \minsp($\C$) $\ge$ $(\alpha + 1)/\alpha$. 
Since the \minsp \,  of
 any clustering for instance $I$ is either 1 or $\alpha + 1$ we have that \minsp($\C$)=$\alpha + 1$

 On the other hand, if the clustering has Min-Sp $\alpha + 1$ then all points in $\X_i$, for each $i$, must be in
 the same group. Moreover, due to the restriction that every $|\X_i| = s_i \in (L/4,L/2)$, we should have
exactly 3 $\X_i$’s in each of the $m$ groups. Thus, the answer is ’YES’.
 \end{proof}

\subsection{Proof of  Theorem \ref{thm:complexity2}}
\label{app:complexity2-proof}

The proof is very similar to that of  Theorem \ref{thm:complexity}.

\begin{proof}
We make a reduction from the $(T/4, T/2)$-restricted 3-PARTITION problem. Given a
multiset $S = \{s_1,\ldots,s_{3m}\}$ of positive integers that satisfies
$\sum_i s_i=mT$, the 3-PARTITION  problem consists of deciding whether or not there exists a partition of $S$ into $m$ triples such that the
sum of the numbers in each one is equal to $T$. In the $(T/4, T/2)$-restricted 3-PARTITION problem,
 there is an additional requirement that each number of $S$ should be in the interval $(T/4, T/2)$. This
 problem is strongly NP-COMPLETE \cite{garey1979computers}

The instance $I = (X, k,L, \dist)$ for our clustering problem is built as follows: we set $L = T$,
$ k = m$; for $i = 1,\ldots, 3m$ let $\X_i$ be a set with $s_i$ points so that the distance between points in
 the same group $\X_i$ is 1/2 while the distance between points in different groups is $\alpha$. We set
$\X = \X_1 \cup \ldots \cup \X_{3m}$.

Let us assume that there is an  $(1, \frac{k-2}{k-1}+\frac{1}{\alpha(k-1)})$-approximation for our problem and let $\C$ be the clustering
 returned by this algorithm for instance $I$. We argue that the answer to the 3-PARTITION problem is ’YES’ if and only if 
 $\spmst(\C)=(k-1)\alpha$.

 First, we show that if the answer is ’YES’, there is a $k$-clustering $\C^*$ for $I $ with \spmst$(\C^*)=(k-1)\alpha$. In
fact, let $S_1,\ldots,S_m$ be a solution of the 3-PARTITION problem and let $\{s_{i_1} , s_{i_2} ,s_{i_3} \}$ the numbers
in $S_i$. Let $\C^*$ be a $k$-clustering where the $i$th group is comprised by all points in $\X_{i_1} \cup \X_{i_2} \cup \X_{i_3}$.
Clearly, each group has $L$ points and the \spmst \, of this clustering is $(k-1)\alpha $. Since our algorithm is a $(1, \frac{k-2}{k-1}+\frac{1}{\alpha(k-1)})$-approximation it returns a clustering 
$\C$ with \spmst($\C$) $>$ $ (k-2)\alpha+1$. 
Since the \spmst \,  of
 any clustering for instance $I$ is either $(k-1) \alpha$  or 
 at most $(k-2) \alpha+1/2$ 
  we have that \spmst($\C$)=$(k-1) \alpha$

 On the other hand, if the clustering has Min-Sp $(k-1) \alpha$ then all points in $\X_i$, for each $i$, must be in
 the same group. Moreover, due to the restriction that every $|\X_i| = s_i \in (L/4,L/2)$, we should have
exactly 3 $\X_i$’s in each of the $m$ groups. Thus, the answer is ’YES’.

\end{proof}

\section{Experiments: Additional Information}
\label{app:experiments}

Experiments were run in an Ubuntu 20.04.5 LTS with 40 cores and 115 GB RAM. The repository of the project at \url{https://anonymous.4open.science/r/SizeConstrainedSpacing-B260} contains the code and instructions needed to generate the experimental data analyzed in the paper.

\subsection{\spmst: comparison between empirical results and upper bound of Algorithm \ref{alg:constrained-MST-SP-2}}

As mentioned in Section \ref{sec:experiments}, Algorithm \ref{alg:constrained-MST-SP-2} can in practice obtain clusterings that are much closer to the optimal \spmst \, than the prediction  guaranteed by Theorem \ref{thm:bound-contrained-mstsp}. In Table \ref{tab:comp-upbound}, we present, for each dataset: the average \spmst \, obtained by Algorithm \ref{alg:constrained-MST-SP-2}; 
the upper bound on the \spmst \, given by the sum of the \minsp \, for all partitions found by an execution of the algorithm; the approximation ratio of Algorithm \ref{alg:constrained-MST-SP-2}, given by  its \spmst \, divided by the upper bound;  and the theoretical approximation ratio $1/H_{k-1}$ from Theorem \ref{thm:bound-contrained-mstsp}.

For all 10 datasets, Algorithm \ref{alg:constrained-MST-SP-2} performs significantly better than its theoretical approximation ratio. The smallest gap between theoretical and empirical result occurs for {\tt avila} dataset, in which the algorithm is  21 percentage points closer to the optimal \spmst \; than Theorem \ref{thm:bound-contrained-mstsp} guarantees; On the other extreme, for  dataset {\tt newsgroups}, it actually achieves the best possible \spmst. 
These results increase our confidence that Algorithm \ref{alg:constrained-MST-SP-2} is a good option for finding separated groups.

\begin{table}[]
\caption{Comparison of \spmst \, of Algorithm \ref{alg:constrained-MST-SP-2} with the upper bound given by Lemma \ref{lem:uppbound09May}}
\label{tab:comp-upbound}
\begin{center}
    
\begin{tabular}{cccccc}
\toprule
\multicolumn{1}{l}{} & \multicolumn{1}{l}{k} & \multicolumn{1}{l}{\spmst} & \multicolumn{1}{l}{$\sum_{\ell=2}^k$ \minsp($\A'_{\ell})$} & \multicolumn{1}{l}{Approximation Ratio} & \multicolumn{1}{l}{1/$H_{k-1}$} \\ 
\midrule
anuran               & 10                    & 1.87                         & 2.42                            & 0.77                      & 0.35                                     \\
avila                & 12                    & 0.81                         & 1.48                            & 0.55                      & 0.34                                     \\
collins              & 30                    & 12.42                        & 13.81                           & 0.9                       & 0.33                                     \\
digits               & 10                    & 178.22                       & 201.47                          & 0.88                      & 0.32                                     \\
letter               & 26                    & 5.67                         & 5.76                            & 0.98                      & 0.31                                     \\
mice                 & 8                     & 5.66                         & 7.12                            & 0.79                      & 0.31                                     \\
newsgroups           & 20                    & 19                           & 19                              & 1                         & 0.3                                      \\
pendigits            & 10                    & 217.01                       & 303.37                          & 0.72                      & 0.3                                      \\
sensorless           & 11                    & 1.36                         & 2.23                            & 0.61                      & 0.29                                     \\
vowel                & 11                    & 4.94                         & 5.57                            & 0.89                      & 0.29                                   \\
\bottomrule
\end{tabular}
\end{center}

\end{table}

\subsection{Average size of smallest clusters}

Table \ref{tab:small-clusters} presents the average size of the smallest group generated by Algorithms \ref{alg:constrained-mispacing} and \ref{alg:constrained-MST-SP-2} and $k$-means. Values tend to be close across all  algorithms, and for all iterations of the experiments the smallest group returned by Algorithms \ref{alg:constrained-mispacing} and \ref{alg:constrained-MST-SP-2} is at least as large as the smallest group from the corresponding $k$-means clustering; recall that $L$ is set as $4s/3$, where $s$ is the size of the smallest group produced by $k$-means. In particular, thanks to the rule of iterating from the largest cluster to the smallest when building our $k$-clustering from an $\ell$-clustering, the theoretical possibility that the smallest group induced by Algorithm \ref{alg:constrained-MST-SP-2} is $1/2$ of the desired size does not appear to happen in practice. 

\begin{table}[]
    \caption{Average size of the smallest cluster in a $k$-clustering, per algorithm and dataset.}
\label{tab:small-clusters}

\begin{center}
\begin{tabular}{c|cc||ccc}
\toprule
\multirow{2}{*}{}  & \multicolumn{2}{c||}{Dimensions}      & \multicolumn{3}{c}{Average size of smallest cluster}     \\
     & n & k & Algorithm \ref{alg:constrained-mispacing}  & Algorithm \ref{alg:constrained-MST-SP-2} & $k$-means  \\ \midrule

anuran     & 7,195  & 10 & 264.5 & 266.7 & 264.1  \\
avila      & 20,867 & 12        & 85.3            & 85.4   & 83.2   \\
collins    & 1,000  & 30 & 7.7 & 7.7 & 7.7  \\
digits     & 1,797  & 10  & 96.2 & 96.2 & 92.7  \\
letter     & 20,000 & 26 & 192 &	258.7 & 191.5   \\
mice       & 552   & 8  & 47.6 &	47.6 &	47     \\
newsgroups & 18,846 & 20             & 188.1	& 188.2	& 188.1	   \\
pendigits  & 10,992 & 10  &   483.2	& 476.4 &	463.8 \\
sensorless & 58,509 & 11 & 1,842.5 &	1,725.5 &	1,655.6  \\
vowel      & 990   & 11  & 57.9 & 57.9 & 57.9   \\

\bottomrule

\end{tabular}
\end{center}
\end{table}

\subsection{Fast version of Algorithm \ref{alg:constrained-MST-SP-2}}

As mentioned in Section \ref{sec:experiments}, the $\log{k}$ bound of Theorem \ref{thm:bound-contrained-mstsp} is still valid for Algorithm \ref{alg:constrained-MST-SP-2} if, instead of investigating all values of $\ell$ from 2 to $k$, it considers only the values that can be written as $ \lceil k / 2^t \rceil$, for  $t=0,\ldots,\lfloor \log k \rfloor$. In Table \ref{tab:comparison-algo2} we compare the results of this fast version of the algorithm with those of the full version.

Even considering the overhead of running the single-linkage algorithm, which cannot be avoided for both versions of Algorithm \ref{alg:constrained-MST-SP-2}, we see a reduction of at least 30\% in the algorithm's running time when using the fast version. The loss in terms of \spmst, on the other hand, is  less than 10\% in the worst scenario, and in 5 of the 10 datasets analyzed both versions return the same clustering.

\begin{table}[]
    \caption{\minsp, \spmst \, and execution time for Algorithm \ref{alg:constrained-MST-SP-2}.}
\label{tab:comparison-algo2}

\begin{center}
\begin{tabular}{c|cc||cc||cc}
\toprule
\multirow{2}{*}{}  & \multicolumn{2}{c||}{Dimensions}      & \multicolumn{2}{c}{\spmst} & \multicolumn{2}{c}{Time (seconds)}     \\
     & n & k                & Fast  & Full  & Fast & Full \\ \midrule

anuran     & 7,195  & 10 & 1.71  & {\bf 1.87}            & \textbf{2.55}   & 6.77   \\
avila      & 20,867 & 12        & 0.77            & \textbf{0.81}   & {\bf 20.47} & 62.51   \\
collins    & 1,000  & 30 & \textbf{12.42}  & \textbf{12.42}  & {\bf 0.76} & 4.48   \\
digits     & 1,797  & 10  & \textbf{178.22} & \textbf{178.22} & {\bf 0.55} & 1.78  \\
letter     & 20,000 & 26 & 5.66            & \textbf{5.67}   & {\bf 23.44} & 116.44   \\
mice       & 552   & 8  & \textbf{5.66}   & \textbf{5.66}   & {\bf 0.27} & 0.44     \\
newsgroups & 18,846 & 20             & {\bf 19}              & \textbf{19}     & {\bf 307.58} & 440.38   \\
pendigits  & 10,992 & 10     & 215.11 & \textbf{217.01} & {\bf 6.13}  &  16.02 \\
sensorless & 58,509 & 11 & 1.34           & \textbf{1.36}   & {\bf 274.00} & 701.42  \\
vowel      & 990   & 11  & \textbf{4.94}   & \textbf{4.94}   & {\bf 0.18} & 0.59   \\

\bottomrule

\end{tabular}
\end{center}
\end{table}

\subsection{Distribution of results for \minsp \, and \spmst}

Figures \ref{fig:boxes_min} and \ref{fig:boxes_mst} show the boxplots for the \minsp \, and the \spmst, respectively, per dataset and algorithm. Algorithm \ref{alg:constrained-MST-SP-2} presents some large variations, when compared to both k-means and Algorithm \ref{alg:constrained-mispacing}, in terms of \minsp \, (Figure \ref{fig:boxes_min}) for some datasets; as it is designed to maximize the \spmst,  this behavior in the other metric presented here should not be too concerning. Also in terms of \minsp, both algorithms presented in the paper clearly outperform k-means in almost all datasets, even considering the variation in results. The same can be said for the \spmst \,(Figure \ref{fig:boxes_mst}), in which, additionally, the range of results returned by Algorithm \ref{alg:constrained-MST-SP-2} is much more in line with those returned by the two other algorithms. 

\begin{figure}
     \centering
     \begin{subfigure}[b]{\textwidth}
         \centering
         \includegraphics[width=\textwidth]{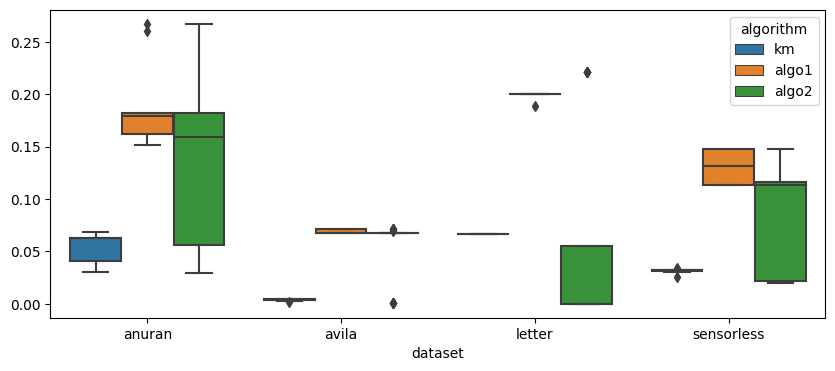}
         \label{fig:box_min_1}
     \end{subfigure}
     \hfill
     \begin{subfigure}[b]{\textwidth}
         \centering
         \includegraphics[width=\textwidth]{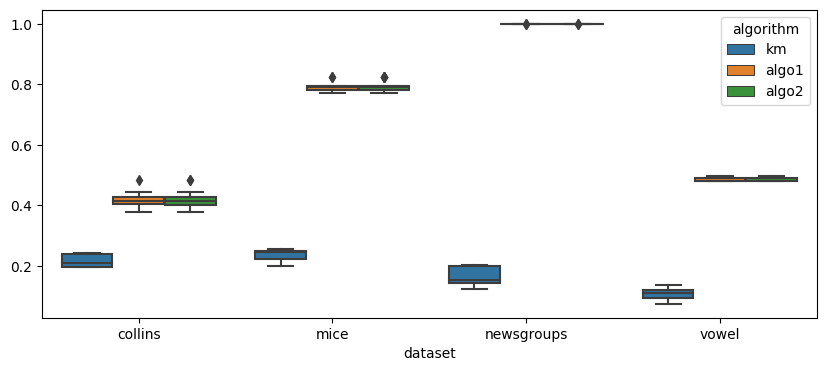}
         \label{fig:box_min_2}
     \end{subfigure}
     \hfill
     \begin{subfigure}[b]{\textwidth}
         \centering
         \includegraphics[width=\textwidth]{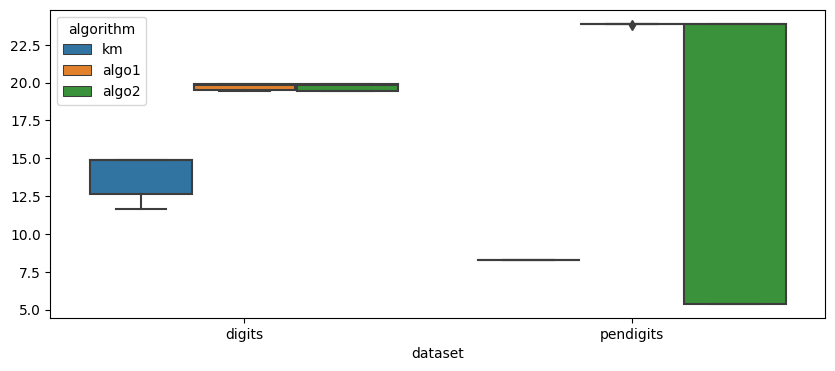}
         \label{fig:box_min_3}
     \end{subfigure}
        \caption{Boxplots of the \minsp \, per dataset and algorithm.}
        \label{fig:boxes_min}
\end{figure}

\begin{figure}
     \centering
     \begin{subfigure}[b]{\textwidth}
         \centering
         \includegraphics[width=\textwidth]{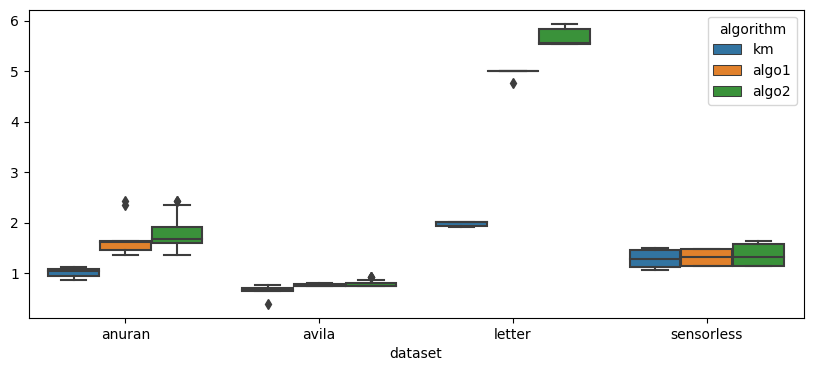}
         \label{fig:box_mst_1}
     \end{subfigure}
     \hfill
     \begin{subfigure}[b]{\textwidth}
         \centering
         \includegraphics[width=\textwidth]{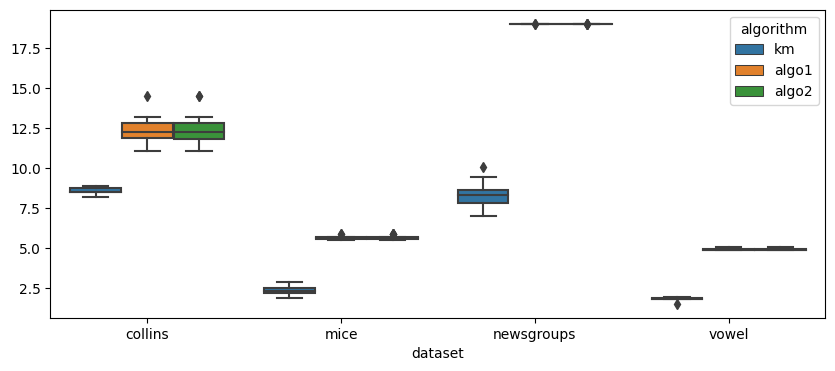}
         \label{fig:box_mst_2}
     \end{subfigure}
     \hfill
     \begin{subfigure}[b]{\textwidth}
         \centering
         \includegraphics[width=\textwidth]{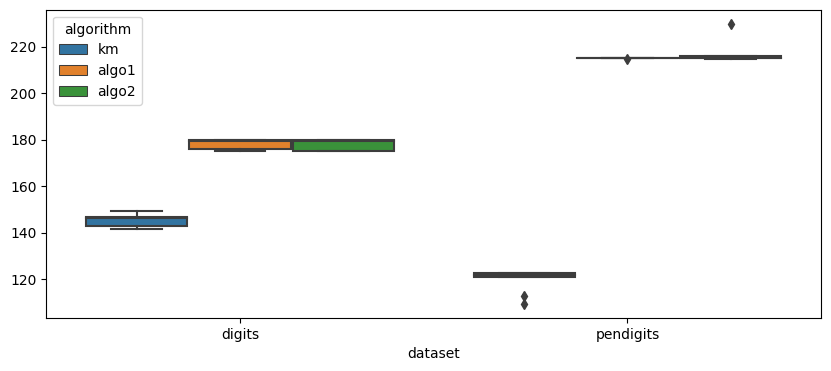}
         \label{fig:box_mst_3}
     \end{subfigure}
        \caption{Boxplots of the \spmst \, per dataset and algorithm.}
        \label{fig:boxes_mst}
\end{figure}

\subsection{Trade-off between size of smallest cluster and inter-group separability criteria}

In Figure \ref{fig:tradeoff} we present scatterplots for our 5 smallest datasets showing how the quality of the clusterings generated by Algorithms \ref{alg:constrained-mispacing} and \ref{alg:constrained-MST-SP-2} (considering, respectively, the \minsp \, and the \spmst \, as criteria) increases as we allow for clusters of smaller sizes. For all datasets,
as expected, allowing for smallest clusters leads to higher \minsp \, and \spmst.  
It is still noteworthy that the algorithms presented in this paper can be used not only to find a good partition with a hard limit on the size of the smallest cluster, but also to find the best balance between minimum size and a good separation of clusters.

\begin{figure}
     \centering
     \begin{subfigure}[b]{0.45\textwidth}
         \centering
         \includegraphics[width=\textwidth]{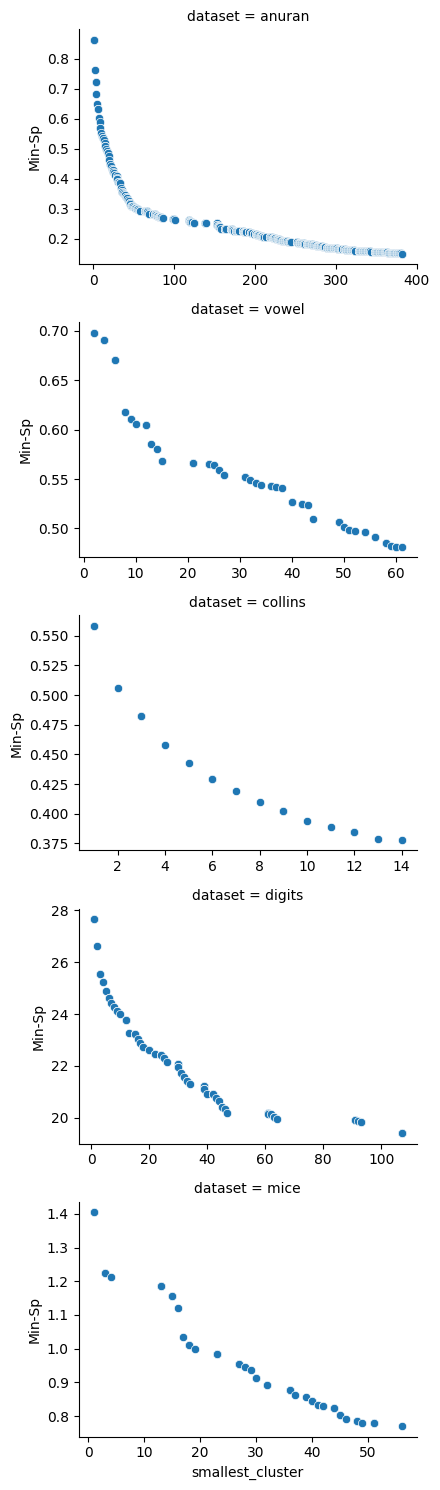}
         \label{fig:min_dist_tradeoff}
         \caption{Algorithm \ref{alg:constrained-mispacing}:\\ smallest cluster vs. \minsp.}
     \end{subfigure}
     \begin{subfigure}[b]{0.45\textwidth}
         \centering
         \includegraphics[width=\textwidth]{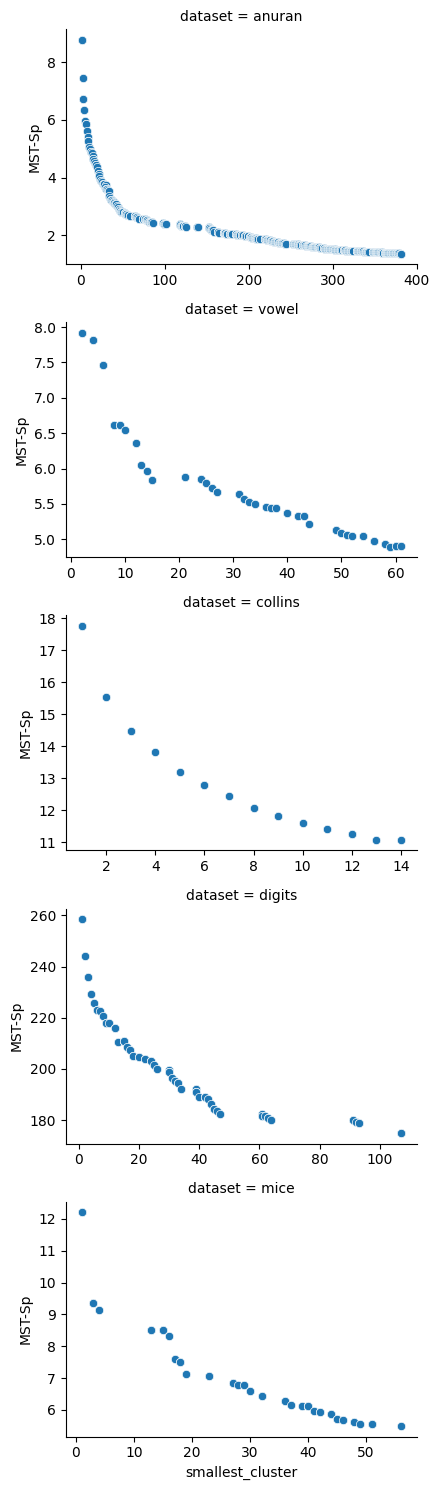}
         \label{fig:mst_cost_tradeoff}
         \caption{Algorithm \ref{alg:constrained-MST-SP-2}:\\ smallest cluster vs. \spmst.}
     \end{subfigure}
        \caption{Trade-off between the size of the smallest cluster and the separability criteria.}
        \label{fig:tradeoff}
\end{figure}

\subsection{Effect of randomness on Algorithm \ref{alg:constrained-MST-SP-2}'s results}

While Algorithm \ref{alg:constrained-mispacing} is fully deterministic, in Algorithm \ref{alg:constrained-MST-SP-2} the split of clusters from an $\ell$-clustering to turn it into a $k$-clustering is performed randomly. In practice, however, this does not affect the results of the algorithm.

For each dataset, we ran 10 seeded iterations of Algorithm \ref{alg:constrained-MST-SP-2} for each value of $L$ used in the experiments. We then calculate the standard deviation of the \spmst \, for each value of $L$. As shown in Table \ref{tab:mst_cost_std}, for 8 of the datasets analyzed the \spmst \, of the clustering returned by Algorithm \ref{alg:constrained-MST-SP-2} is always the same for a given value of $L$; for {\tt letter} and {\tt sensorless}, there is some variation, but it is very small compared to the average \spmst \, returned by the algorithm.

\begin{table}[]
    \caption{Maximum standard deviation of \spmst \, for Algorithm \ref{alg:constrained-MST-SP-2}.}
\label{tab:mst_cost_std}

\begin{center}
\begin{tabular}{c|cc||cc}
\toprule
\multirow{2}{*}{}  & \multicolumn{2}{c||}{Dimensions} & \multicolumn{2}{c}{\spmst} \\
     & n & k & $\mu$ & $\max{\sigma}$ \\ \midrule

anuran     & 7,195  & 10 & 1.87 & -   \\
avila      & 20,867 & 12 & 0.81 & - \\
collins    & 1,000  & 30 & 12.42 & - \\
digits     & 1,797  & 10 & 178.22 & - \\
letter     & 20,000 & 26 & 5.67 & 0.34 \\
mice       & 552   & 8  & 5.66 & - \\
newsgroups & 18,846 & 20 & 19 & - \\
pendigits  & 10,992 & 10 & 217.01 & - \\
sensorless & 58,509 & 11 & 1.96 & 0.002  \\
vowel      & 990   & 11 & 4.94 & - \\

\bottomrule

\end{tabular}
\end{center}
\end{table}

\subsection{Relative quadratic loss for Algorithms \ref{alg:constrained-mispacing} and \ref{alg:constrained-MST-SP-2}}

In Table \ref{tab:quadloss} we present the quadratic loss of both Algorithm \ref{alg:constrained-mispacing} and Algorithm \ref{alg:constrained-MST-SP-2} as a multiple of the loss incurred by the $k$-means algorithm, which is specifically designed to minimize this loss. As expected, since both algorithms were devised for maximizing inter-group criteria, they perform poorly in light of this intra-group loss function --- with the sole exception of the {\tt newsgroups} dataset, for which both algorithms incur a loss only 5\% above that of $k$-means. Across datasets, the performance of both algorithms is similar for this loss, with only small variations.

\begin{table}[]
    \caption{Quadratic loss ($k$-means cost) for Algorithms \ref{alg:constrained-mispacing} and \ref{alg:constrained-MST-SP-2}.}
\label{tab:quadloss}

\begin{center}
\begin{tabular}{c|cc||cc}
\toprule
\multirow{2}{*}{}  & \multicolumn{2}{c||}{Dimensions} & \multicolumn{2}{c}{Loss (relative to $k$-means)}     \\
     & n & k                & Algorithm \ref{alg:constrained-mispacing} & Algorithm \ref{alg:constrained-MST-SP-2} \\ \midrule

anuran     & 7,195  & 10 & 2.50 & \textbf{2.07}   \\
avila      & 20,867 & 12 & \textbf{2.81}  & \textbf{2.81}   \\
collins    & 1,000  & 30 & \textbf{2.33}  & \textbf{2.33}   \\
digits     & 1,797  & 10 & \textbf{1.33} & \textbf{1.33} \\
letter     & 20,000 & 26 & {\bf 2.30}   & 2.52 \\
mice       & 552   & 8  & \textbf{2.52}  & \textbf{2.52} \\
newsgroups & 18,846 & 20 & \textbf{1.05}     & {\bf 1.05}   \\
pendigits  & 10,992 & 10 & 2.18 & \textbf{2.10} \\
sensorless & 58,509 & 11 & {\bf 4.32} & 5.25  \\
vowel      & 990   & 11 & \textbf{2.07}  & \textbf{2.07} \\

\bottomrule

\end{tabular}
\end{center}
\end{table}

\end{document}